\documentclass[11pt]{article}

\usepackage[
  a4paper,
  top=2.5cm, bottom=2.5cm,
  left=2.8cm, right=2.8cm,
  headheight=14pt
]{geometry}

\usepackage[T1]{fontenc}
\usepackage[utf8]{inputenc}
\usepackage{libertine}              
\usepackage[libertine]{newtxmath}   
\usepackage[scaled=0.85]{inconsolata} 
\usepackage{microtype}
\UseMicrotypeSet[protrusion]{basicmath}
\microtypesetup{nopatch=footnote}

\usepackage[table]{xcolor}

\definecolor{accent}{HTML}{1B3A5C}      
\definecolor{accentlight}{HTML}{3D7ABF}  
\definecolor{accentbg}{HTML}{EBF2FA}     
\definecolor{warmgray}{HTML}{F7F5F2}     
\definecolor{darktext}{HTML}{2C2C2C}     
\definecolor{thmbox}{HTML}{E8F4E8}       
\definecolor{lembox}{HTML}{FFF8E1}       
\definecolor{defbox}{HTML}{F3E5F5}       
\definecolor{rembox}{HTML}{FFF3E0}       
\definecolor{rulered}{HTML}{C0392B}      

\usepackage{amsmath,amsfonts}

\usepackage{amssymb}

\usepackage{amsthm}
\usepackage{bm}
\usepackage{bbm}
\usepackage{stmaryrd}
\usepackage{nicefrac}
\usepackage{thmtools}
\usepackage{thm-restate}

\theoremstyle{plain}
\newtheorem{theorem}{Theorem}[section]
\newtheorem{lemma}[theorem]{Lemma}
\newtheorem{proposition}[theorem]{Proposition}

\theoremstyle{definition}
\newtheorem{definition}[theorem]{Definition}
\theoremstyle{remark}
\newtheorem{remark}[theorem]{Remark}

\usepackage[most]{tcolorbox}

\tcolorboxenvironment{theorem}{
  enhanced, breakable,
  colback=thmbox, colframe=accent!60!black,
  boxrule=0.5pt, arc=2pt,
  left=6pt, right=6pt, top=4pt, bottom=4pt,
}
\tcolorboxenvironment{lemma}{
  enhanced, breakable,
  colback=lembox, colframe=accent!50!black,
  boxrule=0.4pt, arc=2pt,
  left=6pt, right=6pt, top=4pt, bottom=4pt,
}
\tcolorboxenvironment{proposition}{
  enhanced, breakable,
  colback=accentbg, colframe=accentlight!70!black,
  boxrule=0.4pt, arc=2pt,
  left=6pt, right=6pt, top=4pt, bottom=4pt,
}
\tcolorboxenvironment{corollary}{
  enhanced, breakable,
  colback=thmbox!50, colframe=accent!40!black,
  boxrule=0.4pt, arc=2pt,
  left=6pt, right=6pt, top=4pt, bottom=4pt,
}
\tcolorboxenvironment{definition}{
  enhanced, breakable,
  colback=defbox, colframe=accent!40!purple,
  boxrule=0.4pt, arc=2pt,
  left=6pt, right=6pt, top=4pt, bottom=4pt,
}
\tcolorboxenvironment{remark}{
  enhanced, breakable,
  colback=rembox, colframe=rembox,
  boxrule=0pt, arc=2pt,
  borderline west={2pt}{0pt}{rulered!70},
  left=8pt, right=6pt, top=4pt, bottom=4pt,
}
\tcolorboxenvironment{assumption}{
  enhanced, breakable,
  colback=rembox!50, colframe=accent!40,
  boxrule=0.4pt, arc=2pt,
  left=6pt, right=6pt, top=4pt, bottom=4pt,
}
\tcolorboxenvironment{property}{
  enhanced, breakable,
  colback=defbox!50, colframe=accent!30!purple,
  boxrule=0.4pt, arc=2pt,
  left=6pt, right=6pt, top=4pt, bottom=4pt,
}

\usepackage{graphicx}
\usepackage{booktabs}
\usepackage{array}
\usepackage{multirow}
\usepackage{subcaption}
\usepackage{wrapfig}
\usepackage{algorithm}
\usepackage{algorithmic}
\renewcommand{\algorithmicrequire}{\textbf{Input:}}
\renewcommand{\algorithmicensure}{\textbf{Output:}}

\usepackage{setspace}
\usepackage{xspace}
\usepackage{enumitem}
\setlist{itemsep=2pt,parsep=0pt}
\usepackage{mdframed}
\usepackage{pifont}
\usepackage{appendix}

\usepackage{titlesec}

\titleformat{\section}
  {\Large\bfseries\sffamily\color{accent}}
  {\colorbox{accent}{\makebox[2em][c]{\color{white}\thesection}}}
  {0.8em}
  {}
  [\vspace{-0.5ex}{\color{accent}\titlerule[1.5pt]}]

\titleformat{\subsection}
  {\large\bfseries\sffamily\color{accent!80!black}}
  {\thesubsection}
  {0.6em}
  {}

\titleformat{\subsubsection}
  {\normalsize\bfseries\sffamily\color{accent!60!black}}
  {\thesubsubsection}
  {0.5em}
  {}

\titlespacing*{\section}{0pt}{3ex plus 1ex minus .2ex}{2ex plus .2ex}
\titlespacing*{\subsection}{0pt}{2.5ex plus .8ex minus .2ex}{1.2ex plus .2ex}

\usepackage{fancyhdr}
\renewcommand{\headrulewidth}{0.4pt}
\renewcommand{\headrule}{\hbox to\headwidth{\color{accent!30}\leaders\hrule height \headrulewidth\hfill}}
\fancypagestyle{plain}{%
  \fancyhf{}
  \fancyfoot[C]{\small\sffamily\color{accent!60}\thepage}
  \renewcommand{\headrulewidth}{0pt}
}

\usepackage[colorlinks=true,
  linkcolor=accent,
  citecolor=accentlight,
  urlcolor=rulered!80!black
]{hyperref}
\usepackage[round]{natbib}

\usepackage{sty/notations}
\usepackage{sty/notations_ixomd}

\providecommand{\keywords}[1]{%
  \par\vspace{6pt}
  \noindent{\sffamily\bfseries\color{accent} Keywords:}\; #1%
}

\usepackage{authblk}

\title{%
  \vspace{-1cm}%
  {\color{accent}\rule{\textwidth}{2pt}}\\[12pt]%
  {\Huge\bfseries\sffamily\color{accent}%
  Optimal last-iterate convergence\\[4pt]
  in matrix games with bandit feedback\\[4pt]
  using the log-barrier}\\[8pt]%
  {\color{accent}\rule{\textwidth}{2pt}}%
}

\author[1]{C\^ome Fiegel}
\author[2]{Pierre M\'enard}
\author[3]{Tadashi Kozuno}
\author[3]{Michal Valko}
\author[1,4,5]{Vianney Perchet}

\affil[1]{ENSAE Paris -- CREST, France}
\affil[2]{ENS Lyon, France}
\affil[3]{Isara Labs}
\affil[4]{Criteo AI Lab}
\affil[5]{Inria Fairplay, Paris, France}

\date{}

\begin{document}

\maketitle
\thispagestyle{plain}

\begin{tcolorbox}[
  enhanced,
  colback=accentbg,
  colframe=accent,
  boxrule=1pt,
  arc=4pt,
  left=12pt, right=12pt, top=10pt, bottom=10pt,
  fonttitle=\large\bfseries\sffamily\color{accent},
  title=Abstract,
  attach boxed title to top left={yshift=-2mm, xshift=10mm},
  boxed title style={
    colback=white, colframe=accent,
    boxrule=0.8pt, arc=2pt,
    left=4pt, right=4pt, top=1pt, bottom=1pt
  }
]
\noindent
We study the problem of learning minimax policies in zero-sum matrix games. \cite{fiegel25harder} recently showed that achieving last-iterate convergence in this setting is harder when the players are uncoupled, by proving a lower bound on the exploitability gap of $\Omega(t^{-1/4})$. Some
\textit{online mirror descent} algorithms were proposed in the literature for this problem, but none have truly attained this rate yet. We show that the use of a log-barrier regularization, along with a dual-focused analysis, allows this $\tilde{\cO}(t^{-1/4})$ convergence with high-probability. We additionally extend our idea to the setting of extensive-form games, proving a bound with the same rate.
\end{tcolorbox}


\vspace{1em}

\newcommand{\mirr}[2]{\textrm{Mirr}_{\Psi,W}\pa{#1,#2}}
\newcommand{\EG}{\textrm{EG}}

\section{Introduction}

Matrix games model games involving a finite number of players, each taking a single action and subsequently receiving a reward. In the zero-sum case, implicitly two-players, a single reward is considered, which is maximized by one and minimized by the other. In this case, the minimax policies \citep{neumann1928zur} model policies that are optimal against the worst opponents.

The problem of finding such policy can be studied in the context of sequential learning, in which the two players will repeatedly play the game in order to find this optimal policy. A classical way of obtaining this bound is through the regret. Assuming that the players observe the expected outcome of each action, basic methods then allow a $\cO(t^{-1/2})$, improvable to $\cO(t^{-1})$ with optimistic methods \citep{popov_modification_1980, rakhlin_optimization_2013}.

These regret-bounding methods generally only ensure an ergodic convergence: their output is the average of the policies played over time. A significant part of the literature focuses on obtaining a last-iterate convergence instead: a direct convergence of the policies played by each player. Interestingly, this requirement does not deteriorate the guarantee; a problem-independent rate of $\cO(t^{-1})$ is still achievable \citep{kangarshahi_lets_2018}, as well as a problem-dependent linear rate \citep{wei2021last}.

In the bandit setting, the players are only allowed to observe their action and the related outcome at every iteration of the game. In this case, the optimal convergence is of $\cO(t^{-1/2})$ with regret-bounding methods \citep{audibert2009minimax,zimmert2019optimal}. Meanwhile, very few articles studied the last-iterate convergence in the bandit case, with \cite{cai2023uncoupled} first achieving a $\cO\pa{t^{-1/8}}$ rate with high probability and $\cO\pa{t^{-1/6}}$ in expectation. More recently, \cite{fiegel25harder} showed that if the two players are uncoupled and are not allowed to communicate their action, then only a $\cO(t^{-1/4})$ rate is achievable. It proposed two methods with this rate, but none really obtained the optimal rate in the desired setting: one relied on a coupling of the players, while the other one was \textit{not fully last-iterate}, in particular with the guarantees only holding close to some chosen horizon $T$.

\textbf{Main contributions:}  \begin{itemize}
    \item We propose and analyse Algorithm~\ref{alg:log}. It relies on a mirror descent approach, along with a varying regularization using the log-barrier. These two ingredients, when combined with a vastly different analysis focusing on the dual, enable the first real $\tilde{\cO}(t^{-1/4})$ anytime last-iterate convergence in the bandit setting.
    \item We additionally show that this approach can be easily extended to extensive-form games with perfect recall, using a dilation \citep{kroer2015faster} of the same regularizer.
\end{itemize}

\section{Problem formulation}

\paragraph{Zero-sum matrix game:}

Two players, called the min- and the max-player, respectively play actions $a\in\cA$ and $b\in\cB$ and receive a stochastic loss $\ell(a,b)\in[0,1]$: the min-player wants to minimize this loss, while the max-player wants to maximize it. 

The two players are allowed to play stochastically: they choose two mixed policies $\mu\in\Delta_{\cA}:=\left\{\mu,\,\sum_{a} \mu(a)=1\right\}$ and $\nu\in\Delta_{\cB}:=\left\{\nu,\, \sum_{b} \nu(b)=1\right\}$ and optimize their expected outcome:
\[\ell(\mu,\nu)=\E_{a\sim\mu,b\sim\nu}\bra{\ell(a,b)}\, .\]
$W=\Delta_{\cA}\times\Delta_{\cB}$ will denote the set of mixed profiles.

We look to obtain a minimax profile, defined as a profile $(\mu^\star,\nu^\star)\in W$ that satisfies
\[\mu^\star\in\argmin_{\mu\in\Delta_{\cA}} \ell(\mu,\nu^\star) \quad \textrm{and} \quad \nu^\star\in\argmax_{\nu\in\Delta_{\cB}}\ell(\mu^\star,\nu)\, ,\]
whose existence is guaranteed \citep{neumann1928zur}.

The proximity of a profile $(\mu,\nu)$ to the set of minimax profiles can be characterized using the exploitability gap:
\[EG(\mu,\nu)=-\min_{\mu'\in\Delta_{\cA}}\ell(\mu',\nu) + \max_{\nu'\in\Delta_{\cB}}\ell(\mu,\nu')\, .\]
In terms of game theory, the exploitability gap is a better measure of proximity to the minimax than any other distance defined solely in terms of probability distributions, say the total variation (which has no strategic interpretation). Indeed, the exploitability gap is zero if and only if $(\mu,\nu)$ is a minimax profile.

\paragraph{Sequential learning with bandit feedback:}

We assume that at each iteration $t$, both players select policies $\mu^t$ and $\nu^t$, sample two actions $a^t\sim\mu^t$ and $b^t\sim\nu^t$, and observe a loss $\ell^t\sim \ell(a^t,b^t)$ associated to these two moves. 

More formally, a filtration $\cF=(\cF^t)_{t\in\N}$ is defined recursively by the observations, along with some extra randomness $\omega$:
\[\cF^t=\sigma\pa{\omega,a^1,b^1,\ell^1 ... , a^{t},b^{t},\ell^{t}}\, ,\]
and the sequence of profiles $(\mu^t,\nu^t)=w^t\in W^\N$ will need to be predictable with respect to $\cF$.

Theorem 5.1 of \cite{fiegel25harder} shows that, in this setting, a convergence to $0$ of $EG(w^t)$ cannot be guaranteed at a rate better than ${\cO}(t^{-1/4})$ if the game is unknown and the two players are not allowed to communicate their own action.

\section{Characterization as a variational inequality problem}

This problem can be usefully formulated as a variational inequality, but this requires first introducing the concept of the \textit{pseudo-gradient}.
\begin{definition}
    The \textbf{pseudo-gradient} of $\ell$ is defined as
    \begin{align*}
    F:W&\xrightarrow[]{} \R^{K}\\
    (\mu,\nu)&\mapsto \pa{\nabla_{\mu} \ell(\mu,\nu),1-\nabla_{\nu}\ell(\mu,\nu)}\, .
\end{align*}
\end{definition}

\begin{definition} \citep{stampacchia_guido_formes_1964}
    Given an operator $F:W\xrightarrow{}\R^K$, with $W\subset \R^K$, we define the (weak) \textbf{variational inequality} problem as finding some $w^\star\in W$ that satisfies
    \[\forall w\in W,\quad\scal{F(w^\star)}{w^\star-w}\leq 0\]
\end{definition}

As shown by the next lemma, this problem can be seen as a generalization of the problem of finding a minimax profile in our setting.

\begin{property}\label{lemma:vi_equiv}
    Since the loss $\ell$ is convex-concave, then the following holds:
    \[w^\star \textrm{ is a Nash equilibrium} \Leftrightarrow w^\star \textrm{ solves the variational inequality above}\, .\]
    Furthermore, since $\ell$ is  actually linear in both of its components,
    \[\EG(w)=\max_{w'\in W}\scal{F(w)}{w-w'}\]
\end{property}

However, as the problem lies in the bandit setting, $F$ is not observed directly when playing. A standard solution is to introduce and to define importance-sampling estimates.

\begin{definition}
    The importance-sampling estimate of the loss of each player at iteration $t$ is defined by
    \[\hell^t_{\textrm{min,IS}}(a)=\frac{\ell^t}{\mu(a)^t}\indic{a=a^t}\quad\textrm{and}\quad \hell^t_{\textrm{max,IS}}(b)=\frac{1-\ell^t}{\nu(b)^t}\indic{b=b^t}\]
    where $a^t$ and $b^t$ are the actions sampled at iteration $t$, and $\ell^t$ the loss observed by both players.
    
    The importance-sampling estimate of the operator $F$ can then be estimated at each iteration at $w^t$ using
    \[\hat{F}(w^t)=\pa{\hell^t_{\textrm{min,IS}},\:\hell^t_{\textrm{max,IS}}}\]
    This estimate is unbiased as long as $w_i^t>0$ for all actions $i$, i.e. :
\[\E_t\bra{\hat{F}(w^t)}=F(w^t)\, .\]
where $\E_t\bra{\cdot}:=\E\bra{\cdot |\cF^t}$.
\end{definition}

\begin{remark}\label{rmk:unbounded}
    This estimator is unbiased, yet it is not bounded, and its variance can become arbitrarily large. Indeed,
    \begin{align*}
        &\E_t\bra{\norm{\hat{F}^{t}(w^t)}^2_2}\\
        &\quad=\sum_{a} \frac{1}{\pa{w_a^t}^2}\E_t\bra{\pa{\ell^t}^2 \middle| a=a^t}\P_t\bra{a=a^t}+\sum_{b} \frac{1}{\pa{w_{A+b}^t}^2}\E_t\bra{\pa{1-\ell^t}^2 \middle| b=b^t}\P_t\bra{b=b^t}\\
        &\quad=\sum_{a} \frac{1}{w_a^t}\E_t\bra{\pa{\ell^t}^2 \middle| a=a^t}+\sum_{b} \frac{1}{w_{A+b}^t}\E_t\bra{\pa{1-\ell^t}^2 \middle| b=b^t}
    \end{align*}
    and this quantity diverges as the probability associated with any action with a non-surely null loss goes to $0$. 
\end{remark}

\section{Intuition of the approach}\label{sec:strong_convexity}

As the game is zero-sum, $F$ is  \textsl{monotone} \citep{martinet_breve_1970,rockafellar_ralph_tyrrell_monotone_1976}, i.e.,  it satisfies the following property:
\[\forall w,w'\in W, \scal{F(w)-F(w')}{w-w'}\geq 0\, .\]
This property explains why the game is solvable; however, obtaining last-iterate convergence is not straightforward, especially in the bandit setting. A method previously studied to obtain it consists in adding some regularization to the operator $F$. Given some Legendre  function $\Psi$ \citep{urruty2001fundamental}, the operator $F_\tau$ is  defined, given $\tau>0$, by
\[F_\tau(w)=F(w)+\tau\nabla\Psi(w)\, .\]
In particular, if $\Psi$ is $1$-strongly convex with respect to some norm $\norm{\cdot}$, then $F_\tau$ becomes $\tau$-strongly monotone, i.e.
\[\forall w,w'\in W, \scal{F_\tau(w)-F_\tau(w')}{w-w'}\geq \tau\norm{w-w'}^2\, .\]
Used with
\[\Psi_{\textrm{entropy}}(w)=-\sum_{i}w_i\log(w_i)\]
which is $1$-strongly convex with respect to the $\norm{\cdot}_1$ norm and combined with mirror descent, this regularization was proven to guarantee a (primal) convergence of the exploitability gap through a convergence of the iterates $w^t$ to the regularized solution $w^{\star,\tau}$ of $F_\tau$. The exploitability gap was then computed as a result, with an additional dependency on $\tau$.

We propose here another way of bounding the exploitability gap, by considering instead the dual norm $\norm{F_\tau(w^t)}_\star$ of the regularized operator on the iterates. We start by providing the intuition by using the Euclidean norm as a regularizer (which is $1$-strongly-convex with respect to itself).

\paragraph{Gradient descent:} With $\Psi=\frac{1}{2}\norm{\cdot}^2_2$ the Euclidean norm, the mirror descent reduces to the gradient descent. Given a constant learning rate $\eta$, ignoring the constraints:
\[w^{t+1}=w^t-\eta^t\hat{F}_\tau(w^t)=w^t-\eta^t\pa{\hat{F}(w^t)+\tau w^t}\, .\]

We further assume here that the unbiased estimate $\hat{F}$ of $F$ has a bounded second-order:
\[\E_t{\norm{\hat{F}(w^t)}_2^2}\leq \sigma^2\quad\textrm{and by extension}\quad \E_t{\norm{\hat{F}_\tau(w^t)}_2^2}\leq \sigma_{\tau}^2:=(\sigma +\tau)^2\, .\]
Note that this assumption is not realistic as demonstrated by Remark~\ref{rmk:unbounded}.

The next proposition aims to illustrate this method of bounding in the dual, in this Euclidean setting for simplicity, relying on this assumption. A more formal result, using the log-barrier, will be stated in the next section.

\begin{proposition}(Informal)
\label{thm:basic}

Let $\tau_0=\tau t^{-1/4}$ and $\eta^t=\frac{1}{\tau}t^{-3/4}$. Then, with the above gradient descent and second-order assumption on the noise:
\[\E\bra{\norm{F(w^t)}}_2=\cO(t^{-1/4})\]
    
\end{proposition}
\begin{proof} (Informal)
At every iteration $t$, with $L=K+\tau^t$:
    \begin{align*}
    \E_t&\bra{\norm{F_{\tau^t}(w^{t+1})}^2_2}-\norm{F_{\tau^t}(w^t)}_2^2\\
    &=2\scal{\E_t\bra{F_{\tau^t}\pa{w^{t+1}}}-F_{\tau^t}(w^t)}{F_{\tau^t}(w^t)}_2+\E_t\bra{\norm{F_{\tau^t}(w^{t+1})-F_{\tau^t}(w^t)}_2^2}\\
    &\leq\frac{-2}{\eta^t}\scal{F_{\tau^t}\pa{\E_t\bra{w^{t+1}}}-F_{\tau^t}(w^t)}{\E_t\bra{w^{t+1}}-w^t}_2+L^2\E_t\bra{\norm{w^{t+1}-w^t}^2_2}\\
    &\leq \frac{-2}{\eta^t}{\tau^t}\norm{\E_t\bra{w^{t+1}}-w^t}^2_2+\pa{L\eta^t}^2\E_t\bra{\norm{\hat{F}_{\tau^t}(w^t)}^2_2}\\
    &\leq -2{\tau^t}\eta^t\norm{F_{\tau^t}(w^t)}_2^2+\pa{L\eta^t\sigma_{\tau}}^2\\
    &= \frac{-2}{t}\norm{F_{\tau^t}(w^t)}_2^2 + \pa{L\sigma_{\tau}/\tau}^2 t^{-3/2}
\end{align*}
where we used, in this order, the unbiasedness of $\hat{F}$ along with the definition of $w^{t+1}$, the $L$-Lipschitzness of $F_{\tau^t}$, the ${\tau^t}$-monotonicity of $F_{\tau^t}$ and the second order assumption on $\hat{F}_{\tau^t}$.

Simultaneously, 

\begin{align*}
    \norm{F_{\tau^{t+1}}(w^{t+1})}^2_2-\norm{F_{\tau^t}(w^{t+1})}^2_2&=2(\tau^{t+1}-\tau^t)\scal{w^{t+1}}{F_{\tau^t}(w^{t+1})}_2+\pa{\tau^{t+1}-\tau^t}^2\norm{w^{t+1}}^2_2\\
&\leq \frac{\tau}{2}t^{-5/4} \norm{F_{\tau^t}(w^{t+1})}_2+\frac{\tau^2}{16}t^{-5/2}
\end{align*}
where we used
\[\abs{\tau^{t+1}-\tau^t}\leq \frac{\tau}{4}t^{-5/4}\]

These two inequalities together lead to a recursive inequality of the type:
\[\E_{t+1}\bra{\norm{F_{\tau^{t+1}}(w^{t+1})}^2_2}\leq \pa{1-\frac{1}{t}}\norm{F_{\tau^t}(w^t)}^2_2+\cO(t^{-3/2})\]
which leads to a bound, at every $t$,
\[\norm{\E\bra{F_{\tau^t}(w^t)}}^2_2= \cO(t^{-1/2})\, .\]
The final bound is then obtained using
\[\norm{F(w^t)}_2\leq \norm{F_{\tau^t}(w^t)}_2+\tau^t\, .\]
\end{proof}

 As previously mentioned, this result could not be applied to our setting mainly because the bounded second-order assumption does not hold on $\hat{F}(w^t)$, at least with the Euclidean norm.

\section{Main algorithm}

Given some Legendre regularizer $\Psi$, the mirror descent is defined iteratively by, given $\eta^t$ as the learning rates and $\xi^t$ as some steps:
\[w^{t+1}=\argmin_{w\in W} D_{\Psi}(w,w^t)+\eta^t\scal{\xi^t}{w}\]
where $D_{\Psi}(w,w')=\Psi(w)-\Psi(w')-\scal{\nabla\Psi(w')}{w-w'}$ is the Bregman divergence \citep{urruty2001fundamental}.

As in the previous section, we propose to use as the step the regularized importance sampling estimate, given some regularization strength $\tau^t$:
\[\xi^t=\hat{F}_{\tau^t}(w^t)=\hat{F}(w^t)+\tau^t\nabla\Psi(w^t)\, .\]

Algorithm~\ref{alg:log}, presented from the point of view of the min-player, uses this approach with the log-barrier as $\Psi$, defined by:
\[\Psi_{\textrm{log-barrier}}(w)=-\sum_{i}\log(w_i)\, .\]

\begin{algorithm}[t]
\caption{Regularized Mirror Descent with log-barrier - Min-player}
\label{alg:log}
\begin{algorithmic}[1]
                \STATE \textbf{Input:} Learning rate $\eta$\\
                Regularization strength $\tau$\\
                Starting iteration $T_0$\\
                Confidence $\delta$
                \STATE \textbf{Define:} $\Psi_{\textrm{min}}(\mu)=-\sum_{i=1}^A \log(\mu_i)$, $\eta^t\gets \eta.(t+T_0)^{-3/4}$ and $\tau^t\gets \tau.\log(\frac{t+T_0}{\delta}) (t+T_0)^{-1/4}$\\
                $\mu^0$ is initialized to the uniform policy.
                \STATE \textbf{Algorithm:} 
                 \textbf{For} $t=0$ to $+\infty$:\\
                ~~~~ \textbf{Sample} and \textbf{play} action $a^t\sim \mu^t$\\
                ~~~~ \textbf{Observe} loss $\ell^t$\\
                ~~~~ $\mu^{t+1} \gets \argmin_{\mu\in\Delta_{\cA}} \bra{D_{\Psi_\textrm{min}}(\mu,\mu^t)+\eta^t \scal{\hat{\ell}_{\textrm{min}}^t}{\mu}}$\\
                where $\hat{\ell}_{\textrm{min}}^t\gets\frac{\ell^t}{\mu^t(a^t)}\indic{a^t}+\tau^t\nabla\Psi_\textrm{min}(\mu^t)$
\end{algorithmic}
\end{algorithm}

It enjoys a convergence with high probability at a near-optimal rate, proven in the appendix, relying on some assumption on the parameters.

\begin{restatable}{assumption}{assparameter}\label{ass:parameters}
We assume that the parameters $\eta$, $\tau$ and $T_0$ of the algorithm satisfy
    \begin{itemize}
    \item $\frac{1}{\eta\tau}=o\:(1)$
    \item $\eta=o\: (1)$
    \item $\eta\tau\leq \frac{T_0}{\log(T_0)^4}$ and $T_0\leq \log(1/\delta)^2\tau^4$
\end{itemize}
Given the first two parameters $\eta$ and $\tau$, the existence of some $T_0$ that satisfies the third assumption is always guaranteed.
\end{restatable}

For clarity reasons in the proof, the exact constants are not computed, but only depend on $K$ and $\delta$. These constants would be high in theory, but we conjecture that they would be reasonable in practice.

\begin{restatable}{theorem}{thmlast}\label{thm:last}
    If both players run Algorithm~\ref{alg:log} with parameters satisfying Assumption~\ref{ass:parameters}, then for any confidence $\delta>0$, with probability  at least $1-\delta$:
    \[\EG(w^t)\leq 2K\tau\log\pa{\frac{t+T_0}{\delta}}(t+T_0)^{-1/4}.\]
    at every iteration $t$.
\end{restatable}

The improvements over the past approaches are mostly achieved through two means:
\begin{itemize}
    \item The different analysis based on the convergence of the dual norm of $F_{\tau^t}(w^t)$ instead of the explicit convergence to some regularized solution. It allows the use of a varying $\tau^t$, which leads to true last-iterate results not relying on a fixed known horizon.
    \item The use of the log-barrier as a regularizer, which is more "stable" than the Shannon entropy and which allows stating the results with high probability rather than in expectation. The main difference can be seen in the quantity $\norm{\hat{F}(w^t)}^2_{\nabla^2\Psi(w^t)^{-1}}$ linked to the variation of the primal iterates, which is bounded almost surely at every iteration with the log-barrier, and only in expectation with the Shannon entropy.
\end{itemize}

\section{Analysis of the Algorithm}

We present in this section the main aspects of the proof of Theorem~\ref{thm:last}. The lemmas specifically are proven in Section~\ref{sec:proof_main_lemmas} of the appendix

We start by defining some norms directly linked to the regularizer, but at specific points of the domain.

\begin{definition}
    For any $w\in \R_{>0}^K$, let $\scal{\cdot}{\cdot}_w$ and $\scal{\cdot}{\cdot}_{\star,w}$ be the scalar products defined by
\[\scal{x}{y}_{w}=\scal{x}{\nabla^2\Psi(w)y} \quad\textrm{and}\quad\scal{x}{y}_{\star,w}=\scal{x}{\nabla^2\Psi(w^t)^{-1}y} \]
with the associated norms
\[\norm{x}_{w}=\sqrt{\scal{x}{\nabla^2\Psi(w)x} }\quad\textrm{and}\quad\norm{x}_{\star,w}=\sqrt{\scal{x}{\nabla^2\Psi(w)^{-1}x} }\,.\]
\end{definition}

These norms, defined locally, play the role of the $\norm{\cdot}_2$ norm of the Section~\ref{sec:strong_convexity} and now really satisfy the assumption on the second-order moment.

\begin{restatable}{lemma}{lemmalipschitz}\label{lemma:lipschitz}
    For every $w\in W$, $F$ is $\sqrt{K}$-Lipschitz with respect to these two norms, i.e:
    \[\forall (w^1,w^2)\in W^2, \norm{F(w^1)-F(w^2)}_{\star,w}\leq \sqrt{K}\norm{w^1-w^2}_w\]
    
    The unbiased importance-sampling estimate $\hat{F}$ of $F$ at $w$ satisfies almost surely:
    \[\quad\norm{\hat{F}(w)}_{\star,w}\leq 2\]
\end{restatable}

To deal with the constraints, the norm of $F_\tau(w)$ is not considered directly, and a more precise quantity is used instead (equivalent without the constraints as the normal cone would be reduced to $\{0\}$).

\begin{definition}
We define, for any $w\in W\cap\R_{>0}^K$, the normal cone
\[N_W(w)=\left\{f\in\R^K \middle| \forall w'\in W \scal{f}{w-w'}\geq 0\right\}\]
and, given $\tau\geq 0$,
\[d_\tau(w)=d_{\norm{\cdot}_{\star,w}}(N_W(w), -F_\tau(w))\]
\end{definition} 

This quantity can be used to bound the exploitability gap at every iteration.

\begin{restatable}{lemma}{lemmadualgap}\label{lemma:dual_gap}
    Assume that $w\in W\cap\R_{>0}^K$ satisfies, for $\tau\geq 0$:
    \[d_\tau(w)\leq \tau\, .\]
    Then
    \[\EG(w)=\max_{w'\in W} \scal{F(w)}{w-w'}\leq 2\tau K\]
\end{restatable}

It justifies using a recursive bound on $d_{\tau^t}(w^t)$ similar to the one of Proposition~\ref{thm:basic}, replacing the gradient descent updates with the mirror descent updates. 

However, obtaining the same kind of recursive bound requires solving many issues that have to be simultaneously treated in the analysis:
\begin{itemize}
    \item The updates now rely on some second order approximation, such as
    \[w^{t+1}= w^t-\eta^t\nabla^2\Psi(w^t) \hat{F}_{\tau^t}(w^t) + \cO((\eta^t)^2)\]
    which have to be analyzed more carefully.
    \item Some extra terms appear because of the norm now changing from $\norm{\cdot}_{\star,w^t}$ to $\norm{\cdot}_{\star,w^{t+1}}$ at every iteration.
    \item The constraints ($w$ must remain in the product of simplices $W$) have to be taken into account in the recursive formula.
\end{itemize}

The extra terms are all handled with the next theorem, which formalizes the recursion that appeared in the proof of Proposition~\ref{thm:basic} in addition of providing the bound with high-probability. The iteration indexes are additionally offset by $T_0$ to ensure that some initial conditions are satisfied.

\begin{restatable}{lemma}{lemmamartingale}\label{lemma:martingale}
    Let $\delta\in (0,1)$, $T_0\in\N$, $U$ be a non-negative stochastic process adapted to some filtration $\cF$, and $T_1$ be the stopping time defined by
    \[T_1=\min\left\{t \;\middle|\; U^t>\frac{1}{\sqrt{t+T_0}}\log\pa{\frac{t+T_0}{\delta}}\right\}\, .\]
    Assume that we have for all $t\leq T_1$:
    \[U^0 \leq \frac{1}{\sqrt{T_0}}\log(T_0)\]
    and
    \[\forall t\in\N,t<T_1 \implies  U^{t+1}= \pa{1-\frac{1}{t+T_0}}U^t+b^{t+1}+W^{t+1}\]
    where $b^{t+1}\leq (t+T_0+1)^{-\frac{3}{2}}$ and $(W^t)_{t>0}$ is a stochastic processes adapted to $\cF$ that satisfies
    \[\E_t\bra{W^{t+1}}=0\quad\textrm{and}\quad \abs{W^{t+1}}^2\leq \frac{(t+T_0+1)^{-\frac{3}{2}}}{2}U^t\, .\]
    Then
    \[\P(T_1 < +\infty)\leq \delta\, .\]
\end{restatable}

Its proof is given in Section~\ref{sec:proof_main_lemmas} and is based on a variation of the proof of the Azuma-Hoeffding inequality. Used with $U^t\simeq d_{\tau^t}(w^t)$, this lemma yields a bound of $d_{\tau^t}(w^t)$ over all iterations with a probability at least $1-\delta$, and thus a bound of the exploitability gap. However, the multiple difficulties mentioned earlier lead to a proof that is quite complex due to the high number of additional terms appearing in the recursion.

\section{Extension to extensive-form games}

\textbf{Extensive-form games} \citep{VONSTENGEL1996220} are a generalization of the matrix games, in which players take successive actions instead. For $n$ players, it consists of
\begin{itemize}
    \item A tree of states $\cS$ of height $H$, partitioned for the min- and max-player into respectively the sets of information sets $\cX$ and $\cY$.
    \item Two action sets, $\cA$ and $\cB$ one for each player.
    \item An initial state $s_1\in\cS$ and a state-transition probability kernel $(p_h)_{h\in[H-1]}$ with $p_h:\cS\times\cA\times\cB\rightarrow \Delta(\cS)$ for each $h \in [H-1]$. 
    \item A loss function $(\ell_h)_{h\in[H]}$ with $\ell_h: \cS \times \cA\times\cB \rightarrow [0,1]^n$.
\end{itemize}

Each game proceeds as follows: 
\begin{itemize}
    \item The game starts at depth $1$ and state $s_1$.
    \item At depth $h\in [H]$, every player $i$ observes the information set $x_{i,h}\in\cX_i$ associated with the current state $s_h$, then choose some action $a_h\in \cA$ and $b_h\in \cB$. 
    \item As a result, every player $i$ receives the loss $\ell_{i,h}(s_h,a_h,b_h)$ and if $h<H$, the state transitions to a new state $s_{h+1}\sim p(\cdot|s_h,a_h, b_h)$ in $\cS$.
\end{itemize}

\textbf{Perfect recall:} We further assume that the players perfectly remember their past observations and actions. Consequently, for any action set $x_h$ of depth $h$, there exists a unique path $(x_1,a_1,,...,x_{h-1},a_{h-1},x_h)$ that leads to it for the min-player, and $(y_1,b_1,,...,y_{h-1},b_{h-1},y_h)$ for the max-player.

The same algorithm can be applied under these settings with some small changes. One of the first changes is due to the fact that, in contrast to the matrix setting, the loss is not linear with respect to each player's policy. It is nonetheless linear with respect to each of the following sequence-form policy.

\begin{definition}[Sequence-form policy and pseudo-gradient]\label{def:seq_form} \citep{VONSTENGEL1996220}
    Let $\mu=(\cdot|x)_{x\in\cX}$ be a policy for the min-player. The associated \textbf{sequence-form} policy $\mu_{\xrightarrow[]{}}$ is defined by, for all action sets $x_h$ of depth $h$ and action $a_h$ 
    \[\mu_{\xrightarrow[]{}x_h}(a_h)=\prod_{h'=1}^h \mu(a_{h'}|x_{h'})\]
    where $(x_1,a_1,,...,x_{h-1},a_{h-1},x_h)$ is the unique path that leads to $x_h$. $\nu_{\xrightarrow[]{}}$ is defined similarly for the max-player.
    
    The sequence-form policy profiles $w=(\mu_{\xrightarrow[]{}},\nu_{\xrightarrow[]{}})$ can similarly be defined, and will be directly referred to as $W$, as the expected loss given these policies.

    It can in particular, be used to define the pseudo-gradient $F$ as in the previous sections:
    \begin{align*}
    F:W&\xrightarrow[]{} \R^K\\
    (\mu_{\xrightarrow[]{}},\nu_{\xrightarrow[]{}})&\mapsto \pa{\nabla_{\mu_{\xrightarrow[]{}}}\pa{ \E\bra{\sum_{h=0}^{H-1}\ell_h(\mu_{\xrightarrow[]{}},\nu_{\xrightarrow[]{}})}},H-\nabla_{\nu_{\xrightarrow[]{}}}\pa{\E\bra{\sum_{h=0}^{H-1}\ell_h(\mu_{\xrightarrow[]{}},\nu_{\xrightarrow[]{}})}}}\, .
\end{align*}
\end{definition}

The same importance-sampling estimate can be defined in the extensive-form setting.

\begin{definition}
    An importance-sampling estimate can be defined similarly, using this time the sequence-form policies:
    \[\hell^t_{\textrm{min}}(x_h,a_h)=\frac{\ell^t}{\mu^t_{\xrightarrow[]{}}(x_h,a_h)}\indic{x_h=x_h^t,a_h=a_h^t}\quad\textrm{and}\quad \hell^t_{\textrm{max}}(b)=\frac{1-\ell^t}{\nu_{\xrightarrow[]{}}^t(y_h,b_h)}\indic{y_h=y_h^t,b_h=b_h^t}\, .\]
    The importance-sampling estimate of the operator $F$ can again be estimated with:
    \[\hat{F}(w^t)=\pa{\hell^t_{\textrm{min}},\:\hell^t_{\textrm{max}}}\]
    This estimate is also unbiased as long as $w^t(x,a)>0$ for all states $x$ and actions $a$:
\[\E_t\bra{\hat{F}(w^t)}=F(w^t)\, .\]
\end{definition}

And this estimate satisfies the same properties as before.

\begin{restatable}{lemma}{lemmalipschitzext}\label{lemma:prop_ext}
    Let $K=\abs{A}\abs{X}+\abs{B}\abs{Y}$ be the total number of actions over both trees. 

    The operator $F$ is linear and monotone. For any $w\in W$, it is also $H\sqrt{K}$-Lipschitz with respect to the $\norm{\cdot}_{w}$ and $\norm{\cdot}_{\star,w}$ norms:
        \[\forall (w^1,w^2)\in W^2, \norm{F(w^1)-F(w^2)}_{\star,w}\leq H\sqrt{K}\norm{w^1-w^2}_{w}\]
    and the above unbiased estimate $\hat{F}$ satisfies:
        \[\norm{\hat{F}(w)}_{\star, w}\leq 2H\]
\end{restatable}

\begin{algorithm}[t]
\caption{Extensive-form Regularized Mirror Descent with log-barrier - Min-player}\label{alg:log_ext}
\begin{algorithmic}[1]
                \STATE \textbf{Input:} Learning rate $\eta$\\
                Regularization strength $\tau$\\
                Starting iteration $T_0$\\
                Confidence $\delta$
                \STATE \textbf{Define:} $\Psi_{\textrm{min}}(\mu)=-\sum_{x}\sum_{a} \log(\mu_{\xrightarrow[]{}x}(a))$, $\eta^t\gets \eta.(t+T_0)^{-3/4}$\\ and $\tau^t\gets \tau.\log(\frac{t+T_0}{\delta}) (t+T_0)^{-1/4}$\\
                $\mu^0$ is initialized to the uniform policy.
                \STATE \textbf{Algorithm:} 
                 \textbf{For} $t=0$ to $+\infty$:\\
                ~~~~ For each $h=0$ to $H-1$:\\
                ~~~~~~~~ \textbf{Sample} and \textbf{play} action $a_h^t\sim\mu^t(\cdot|x_h^t)=\frac{\mu^t_{\xrightarrow[]{}x_h^t}(\cdot)}{\mu^t_{\xrightarrow[]{}x_{h-1}^t}(a_{h-1}^t)}$ (or $\mu_{\xrightarrow[]{}x_0}(\cdot)$ if $h=0$)\\
                ~~~~~~~~ \textbf{Observe} loss $\ell_h^t$\\
                ~~~~ \textbf{Update}
                ~~~~ $\mu^{t+1} \gets \argmin_{\mu\in\Delta_{\cA}} \bra{D_{\Psi_\textrm{min}}(\mu,\mu^t)+\eta^t \scal{\hat{\ell}_{\textrm{min}}^t}{\mu}}$\\
                where $\hat{\ell}_{\textrm{min}}^t\gets\sum_{h=0}^{H-1}\frac{\ell_h^t}{\mu_{\xrightarrow[]{}x_h^t}^t(a^t)}\indic{x_h^t,a_h^t}+\tau^t\nabla\Psi_\textrm{min}(\mu^t)$
\end{algorithmic}
\end{algorithm}

For this reason, Algorithm~\ref{alg:log_ext} enjoys the same kind of guarantees as Algorithm~\ref{alg:log}.

\begin{restatable}{theorem}{thmlastext}\label{thm:last_ext}
    If both players run Algorithm~\ref{alg:log_ext} with parameters satisfying Assumption~\ref{ass:parameters}, then for any confidence $\delta>0$, with probability  at least $1-\delta$:
    \[\EG(w^t)\leq 2K\tau\log\pa{\frac{t+T_0}{\delta}}(t+T_0)^{-1/4}.\]
    at every iteration $t$.
\end{restatable}

The proof is indeed the same as in Theorem~\ref{thm:last}, simply using instead the properties of Lemma~\ref{lemma:prop_ext}.

\section{Conclusion}

Algorithms~\ref{alg:log} and \ref{alg:log_ext} solve the main issue behind the approaches of  \cite{fiegel25harder}. Using the more stable log-barrier regularization, it achieves a last-iterate convergence with high probability without initial knowledge of the horizon. The analysis, completely presented in the appendix, relies on an analysis in the dual space. 

It raises the following open questions:

\begin{itemize}
    \item While the main idea of the proof (recursively upper-bounding $d^{\tau^t}(w^t)$) is straightforward, it is quite long as all the above details need to be accounted for. Is a simpler proof achievable?
    \item Is the use of the log-barrier regularization really necessary for the adaptivity of $\tau$ to work? Or can other regularization also be adaptive to the horizon, potentially without the high probability guarantees?
    \item Could an optimistic mirror descent approach work for this problem, and would it be more efficient in practice?
    \item Currently, the updates of Algorithm~\ref{alg:log_ext} are quite expensive as they act on the whole tree at each iteration, implying a time complexity of at least $\Omega(K)$ at each iteration. Does there exist an efficient way of computing them (by only updating along the trajectory when needed for example).
\end{itemize}

\newpage
\appendix

The appendix aims to prove Theorems~\ref{thm:last} and \ref{thm:last_ext}.

\section{Proof - Main lemmas} \label{sec:proof_main_lemmas}

We start by proving the lemmas stated in the main article.

\setcounter{theorem}{9}

\lemmadualgap*

\begin{proof}
    We will use that
    \[\nabla\Psi(w)=(-w_i^{-1})_{i\in[K]} \quad \textrm{and}\quad \nabla^2\Psi(w)=\textrm{Diag}\pa{(w_i^{-2})_{i\in[K]}}\]

    Let $g\in N_W(w)$ such that 
    \[d_\tau(w)=\norm{F_\tau(w)+g}_{\star,w}\, .\] 
    From the assumption, we first notice that, for all $i\in[K]$, by definition of the norm $\norm{\cdot}_{\star,w}$,
    \[w_i^2\abs{F(w)_i+g_i-\tau w_i^{-1}}^2\leq \tau^2\, ,\]
    which is equivalent to
    \[\abs{w_i\pa{F(w)_i+g_i}-\tau}\leq \tau\, ,\]
    and which implies in particular $F(w)_i+g_i\geq 0$.
    
    Then, for any $w'\in W$,
    \begin{align*}
        \scal{F(w)}{w-w'}&\leq\scal{F(w)+g}{w-w'}\\
        &\leq \scal{F(w)+g}{w}\\
        &=\scal{F_\tau(w)+g}{w} + \tau K\\
        &\leq \sqrt{K}d_\tau(w) + \tau K\\
        &\leq 2\tau K
    \end{align*}
    where we used, in this order:
    \begin{itemize}
        \item the definition of the normal cone associated to $w$,
        \item $F(w)+g$ being in the non-negative quadrant as shown above
        \item $\scal{\nabla\Psi(w)}{w}= -K$ by definition of $\nabla\Psi$,
        \item the Cauchy-Schwarz inequality, as 
        \[\pa{\sum_{i=1}^K w_i\pa{F_\tau(w)_i+g_i}}^2\leq K\sum_{i=1}^K w_i^2\pa{F_\tau(w)_i+g_i}^2\]
    \end{itemize}
    
\end{proof}

\setcounter{theorem}{11}
\lemmamartingale*

\begin{proof}
    We define the stochastic process $Q$ adapted to $\cF$ by
    \[Q^t=e^{\sqrt{t+T_0}U^t-\log(t+T_0)}\]
    for all $t\leq T_1$, and show that $(Q^{t\wedge T_1})_t$ is a super-martingale. Take $t<T_1$, we first observe, using Hoeffding Lemma and the assumption on $W$, that:
    \[\E_s\bra{e^{\sqrt{t+T_0+1}W^t}}\leq e^{\frac{U^t}{2\sqrt{t+T_0+1}}}\leq e^{\frac{\sqrt{t+T_0+1}U^t}{2(t+T_0)}}\]
    
    Then,
    \begin{align*}
        \E_s\bra{Q^{t+1}}&=\E_s\bra{e^{\sqrt{t+T_0+1}U^{t+1}-\log(t+T_0+1)}}\\
        &= e^{\sqrt{t+T_0+1}(1-1/(t+T_0))U^{t}}\E_t\bra{e^{\sqrt{t+T_0+1}W^{t+1}}}e^{\sqrt{t+T_0+1}b^{t+1}-\log(t+T_0+1)}\\
        &\leq e^{\sqrt{t+T_0+1}(1-1/(2(t+T_0)))U^t} e^{1/(t+T_0+1)-\log(t+T_0+1)}\\
        &\leq e^{\sqrt{t+T_0}U^t}e^{-\log(t+T_0)}\\
        &=Q^t
    \end{align*}

    where we relied, in this order, on
    \begin{itemize}
        \item the assumption regarding $U^{t+1}$,
        \item the above inequality from Hoeffding Lemma,
        \item the upperbound of $b^{t+1}$,
        \item the inequality $\sqrt{s+1}(1-1/(2s))\leq \sqrt{s}$ for $s=t+T_0$, as
        \[\sqrt{s+1}-\sqrt{s}=\int_{s}^{s+1}\frac{du}{2\sqrt{u}}\leq \int_{s}^{s+1}\frac{du}{2\sqrt{s}}=\frac{1}{2\sqrt{s}}\leq \frac{\sqrt{s+1}}{2s}\, ,\]
        \item the inequality $\frac{1}{s+1}-\log(s+1)\leq -\log(s)$ for $s=T+T_0$, as \[\log(s+1)-\log(s)=\int_{s}^{s+1}\frac{du}{u}\geq \int_{s}^{s+1}\frac{du}{s+1}=\frac{1}{s+1}\, .\]

    \end{itemize}
    This proves that $(Q^{t\wedge T_1})$ is a super-martingale.

    Then, using Ville's inequality
    \begin{align*}
        \P(T_1 < +\infty)&=\P\pa{\exists t\geq T_0,\: U^t>\frac{1}{\sqrt{t+T_0}}\log\pa{\frac{t+T_0}{\delta}}}\\
        &=\P\pa{\exists t\geq T_0,\: \sqrt{t+T_0}U^t-\log(t+T_0)>\log\pa{\frac{1}{\delta}}}\\
        &=\P\pa{\exists s\geq T_0,\: Q^{t\wedge T_1}>1/\delta}\\
        &\leq \E(Q^0)\delta\\
        &\leq \delta
    \end{align*}
    which concludes.    
\end{proof}

\setcounter{theorem}{15}
\lemmalipschitz*
\setcounter{theorem}{17}

\begin{proof}
As for all $i$, $w_i\leq 1$, and, for a given action of one player, the reward is Lipschitz with respect to the policy of the opponent in $L^1$:
\begin{align*}
    \norm{F(w')-F(w)}^2_{\star,w}&=\sum_{i=1}^K (w_i)^2 \abs{F(w')_i-F(w)_i}^2\\
&=\sum_{i=1}^A (w_i)^2 \abs{F(w')_i-F(w)_i}^2+\sum_{i=1+1}^{A+B} (w_i)^2 \abs{F(w')_i-F(w)_i}^2\\
&\leq \sum_{i=1}^A (w_i)^2 \abs{\sum_{j=A+1}^{A+B}\abs{w'_j-w_j}}^2+\sum_{i=A+1}^{A+B} (w_i)^2 \abs{\sum_{j=1}^{A}\abs{w'_j-w_j}}^2\\
&\leq \sum_{i=1}^A (w_i)^2B \sum_{j=A+1}^{A+B}\abs{w'_j-w_j}^2+\sum_{i=A+1}^{A+B} (w_i)^2 A\sum_{j=1}^{A}\abs{w'_j-w_j}^2\\
&\leq K\sum_{j=1}^K \abs{w'_j-w_j}^2\\
&\leq K\sum_{j=1}^K \frac{1}{(w_j)^2}\abs{w'_j-w_j}^2\\
&\leq K\norm{w'-w}^2_w
\end{align*}
hence $L=\sqrt{K}$

As $w^0$ is the product of uniform policies,
\[\norm{F(w^0)}^2_{\star,w}
=\sum_{i=1}^K (w^0_i)^2 \abs{F(w^0)_i}^2
    \leq \sum_{i=1}^K (w^0_i)^2
    = \frac{1}{A}+\frac{1}{B}
    \leq 1\]
as $A\geq 2$ and $B\geq 2$

And, using the important sampling estimator,
\begin{align*}
    \norm{\hat{F}(w)}_{\star,w}&=\sum_{i=1}^K (w_i)^2\abs{\hat{F}(w)_i}^2\\
    &=\sum_{i=1}^A (w_i)^2\abs{\hell_{\textrm{min},i}}^2+\sum_{i=A+1}^{A+B}(w_i)^2\abs{\hell_{\textrm{max},i}}^2\\
    &\leq \sum_{i=1}^A \indic{i}^2+\sum_{i=A+1}^B \indic{i}^2\\
    &=2
\end{align*}
    
\end{proof}

\lemmalipschitzext*

\begin{proof}
    The proof is the same as above, with an extra sum on the depth $h$.
\end{proof}

\section{Proof - Technical lemmas}
We now use the notations, given $W=\Delta_{\cA}\times \Delta_{\cB}$ for the affine subset of $\R_{\geq 0}^K$ with $K=A+B$, $W_0$ for the linear component, and $G$ for its orthogonal. We will use the fact that, in the relative interior of $W$, the dual cone of any point with respect to $W$ is always $G$.

\begin{lemma}
    For all $t\geq T_0$,
    \[\norm{\hat{F}_{\tau^t}(w^t)}_{\star,w^t}\leq \sigma'\]
    where
    \[\sigma'=\sigma+\tau^{0}\sqrt{K}\quad\textrm{and }\sigma \textrm{ is either }2\textrm{ or }2H\]

    In particular,
    \[\norm{p^t}_{\star,w^t}\leq \sigma'\]
\end{lemma}

\begin{proof}
    From the triangular inequality, 
    \begin{align*}
        \norm{\hat{F}_{\tau^t}(w^t)}_{\star,w^t}&\leq \norm{\hat{F}(w^t)}_{\star,w^t}+\tau^t\norm{\nabla\Psi(w^t)}_{\star,w^t}\\
        &\leq \sigma +\tau^{t}\sqrt{\sum_{i=1}^K \pa{\frac{w_i^t}{w_i^t}}^2}\\
        &=\sigma +\tau^t\sqrt{K}
    \end{align*}
    and we conclude using the fact that $(\tau^t)_{t\geq T_0}$ is decreasing.

    The second inequality is a consequence of Jensen inequality and the unbiasedness of the estimator $\hat{F}$.
\end{proof}

\begin{lemma}\label{lemma:delta_reg}
    Assume $T_0\geq e^4$, then $(\tau^t)_t$ is decreasing and
    \[\abs{\tau^t-\tau^{t+1}}\leq \frac{\tau^t}{4(t+T_0)}\leq \frac{\tau^0}{4(t+T_0)}\]
\end{lemma}

\begin{proof}
    Let $h$ be the function defined on $\R_{>0}$
    \[h(u)=\frac{\log(u)}u^{-\frac{1}{4}}\]
    $h$ is differentiable and satisfies
    \[h'(u)=u^{-\frac{5}{4}}-\frac{\log(u)}{4}u^{-\frac{5}{4}}\]
    Both $h$ and its derivative are thus decreasing on $(e^4,+\infty)$, and we obtain
    \begin{align*}
        \abs{\tau^t-\tau^{t+1}}&=\tau\abs{h(t+T_0)-h(t+T_0+1)}\\
        &=\tau\abs{\int_{t+T_0}^{t+T_0+1}h'(u)du}\\
        &\leq \tau\int_{t+T_0}^{t+T_0+1}\abs{h'(u)}du\\
        &\leq \tau\int_{t+T_0}^{t+T_0+1} \frac{\log(u)}{4}u^{-\frac{5}{4}}du\\
        &\leq \tau\int_{t+T_0}^{t+T_0+1} \frac{\log(t+T_0)}{4}(t+T_0)^{-\frac{5}{4}}du\\
        &=\tau \frac{\log(t+T_0)}{4}(t+T_0)^{-\frac{5}{4}}\\
        &=\frac{\tau^t}{4(t+T_0)}\\
        &\leq \frac{\tau^0}{4(t+T_0)}\, ,
    \end{align*}
    where the last property follows the fact the $(\tau^t)_t$ sequence is decreasing if $T_0\geq e^4$.
\end{proof}


\begin{lemma}\label{lemma:double}
    Assume $\sigma'\eta^0\leq \frac{1}{\sqrt{12}}$, then for all $t\geq T_0$
    \[\forall i\in [K], \frac{w_i^t}{2}\leq w_i^{t+1}\leq 2 w_i^t\]
\end{lemma}

\begin{proof}
    Let 
    \[D_\phi(x,y):=h(x/y)\quad\textrm{where}\quad h(r)=r-\log(r)-1\]
    be the component-wise value of the Itakura-Saito divergence. It satisfies the following properties:
    \begin{itemize}
        \item $D_{\Psi}(\mu,\tilde{\mu})=\sum_{i=1}^K D_\phi(\mu_i,\tilde{\mu}_i)=\sum_{i=1}^Kh(\mu_i/\tilde{\mu}_i)$ by definition of the Itakura-Saito divergence
        \item $h(r)\leq (1-r)^2$ for all $r\geq \frac{1}{\sqrt{2}}$ as, from Taylor-Lagrange formula, there exists $\tilde{r}$ between 1 and $r$ such that
        \[h(r)=h(1)+(r-1)h'(1)+\frac{(r-1)^2}{2}h''(\tilde{r})=\frac{(r-1)^2}{2\tilde{r}^2}\leq (r-1)^2\]
        \item $h(r)$ is non-negative for all $r>0$
        \item $h(r)\leq 1/6$ implies $r\in(1/2,2)$, as $h$ is a convex function with $h(1)=0$, $h(1/2)\approx 0.193$ and $h(2)\approx 0.307$.
    \end{itemize}

    Now, let $\tilde{\mu}^{t+1}$ be the result of an unrestricted mirror step at time $t$, defined as the only vector of $\R^K$ that satisfies
    \[\nabla\Psi(\tilde{\mu}^{t+1})=\nabla\Psi(\mu^t)-\eta^t\hat{F}_{\tau^t}(w^t)\]
    This implies
    \[\mu^{t+1}=\argmin_{\mu\in W}D_{\Psi}(\mu,\tilde{\mu}^{t+1})\, .\]
    As such, for all $i\in[K]$,
    \[\frac{\mu_i^t}{\tilde{\mu}_i^{t+1}}=1-\eta^tw_i^t\hat{F}_{\tau^t}(w^t)_i\quad\textrm{and in particular,}\quad \frac{\mu_i^t}{\tilde{\mu}_i^{t+1}}\geq 1-\sigma\eta^t\geq 1-\frac{1}{\sqrt{12}}\geq\frac{1}{\sqrt{2}}\, .\] 
    Then, using the generalized Pythagorean theorem for Bregman divergence, and the previous properties:
    \begin{align*}
        D_{\Psi}(\mu^t,\mu^{t+1})&\leq D_{\Psi}(\mu^t,\tilde{\mu}^{t+1})-D_{\Psi}(\mu^{t+1},\tilde{\mu}^{t+1})\\
        &\leq D_{\Psi}(\mu^t,\tilde{\mu}^{t+1})\\
        &=\sum_{i=1}^K h(\mu^t_i/\tilde{\mu}_i^{t+1})\\
        &\leq \sum_{i=1}^K\pa{\frac{\mu_i^t}{\tilde{\mu}_i^{t+1}}-1}^2\\
        &=\sum_{i=1}^K\pa{\eta^tw_i^t\hat{F}_{\tau^t}(w^t)_i}^2\\
        &\leq (\eta^t\sigma')^2\leq (\eta^{T_0}\sigma')^2\leq \frac{1}{12}\\
    \end{align*}
    as $(\eta^t)_{t\geq T_0}$ is decreasing. In particular, for all $i\in[K]$,
    \[h\pa{\frac{\mu_i^t}{\mu_i^{t+1}}}\leq \frac{1}{12}\]
    and we obtain the lemma.
    
\end{proof}

For any $t\geq T_0$, we define $p^t$ and $g^t$ to be respectively:
\[g^t=\argmin_{g\in G} \norm{F_{\tau^t}(w^t)-g}_{\star,w^t}\quad\textrm{and}\quad p^t=F_{\tau^t}(w^t)-g^t\]
The same is done with $\hat{F}_{\tau^t}(w^t)$, with $\hat{p}^t$ and $\hat{g}^t$:
\[\hat{g}^t=\argmin_{g\in G} \norm{\hat{F}_{\tau^t}(w^t)-g}_{\star,w^t}\quad\textrm{and}\quad \hat{p}^t=\hat{F}_{\tau^t}(w^t)-g^t\]

\begin{lemma}\label{lemma:pythagore}
    For any $w\in W$ and $\xi\in \R^K$, we have the equivalence:
    \[\nabla^2\Psi(w)^{-1}\cdot \xi\in W_0\quad\Longleftrightarrow\quad 0=\argmin_{g\in G} \norm{\xi-g}_{\star,w}\]
    and, for all $w_0\in W_0$
    \[w_0\in W_0\quad\Longleftrightarrow\quad 0=\argmin_{g\in G} \norm{w_0-\nabla^2\Psi(w)^{-1}\cdot g}_{w}\]
\end{lemma}
\begin{proof}
    Pythagorean with $W_0\perp G$
\end{proof}

In particular, this implies that $\nabla^2\Psi(w^t)^{-1}\cdot \hat{p}^t\in W_0$ and $\nabla^2\Psi(w^t)^{-1}\cdot p^t\in W_0$.

\section{Proof - First order approximations of the next iterate}

The following assumption implies that the learning rate $\eta$ is small enough for the second-order terms to be negligible compared to the first-order terms.
\begin{assumption}\label{assum:sigma}
    $32\sigma'\eta^0\leq 1$
\end{assumption}

The idea behind the next lemma is that $\nabla^2\Psi(w^t)^{-1}\hat{p}^t$ and $\nabla^2\Psi(w^t)^{-1}p^t$ multiplied by the learning rate $\eta^t$ respectively approximate $w^{t}-w^{t+1}$ and $w^t-\E_t\bra{w^{t+1}}$ up to some second order terms. For inequality \eqref{eq:3}, we use the previous assumption to upper-bound this second order term by $\sigma'\eta^t$.

\begin{proposition}\label{lemma:w_inequalities}
For all $t\geq T_0$, the following inequalities hold:
    \begin{equation}\norm{w^{t+1}-w^t}_{w^t}\leq 2\sigma'\eta^t\label{eq:3}
    \end{equation}
    \begin{equation}\norm{\E_t\bra{w^{t+1}}-w^t+\eta^t\nabla^2\Psi(w^t)^{-1}p^t}_{w^t}\leq 32(\sigma'\eta^t)^2\label{eq:2}
    \end{equation}
    And, from the triangle inequality,
    \begin{equation}\norm{\E_t\bra{w^{t+1}}-w^t}_{w^t}\leq \eta^t\norm{ p^t}_{\star,w^t}+32\pa{\sigma'\eta^t}^2\, .\label{eq:4}
    \end{equation}
\end{proposition}

\begin{proof}
    We start by showing
    \[\norm{w^{t+1}-w^t+\eta^t\nabla^2\Psi(w^t)^{-1}\hat{p}^t}_{w^t}\leq 32(\sigma'\eta^t)^2\, ,\]
    and the three inequalities will follow.
    
    We first notice, from Lemma~\ref{lemma:pythagore}, that 
    \[\norm{w^{t+1}-w^t+\eta^t\nabla^2\Psi(w^t)^{-1}\hat{p}^t}_{w^t}\leq \norm{w^{t+1}-w^t+\nabla^2(w^t)^{-1}\pa{\nabla\Psi(w^t)-\nabla\Psi(w^{t+1})}}_{w^t}\]
    as $\hat{p}^t-\nabla\Psi(w^t)+\nabla\Psi(w^{t+1})\in G$
    
    Now, let $J_i^t$ be the interval of values between $w_i^t$ and $w_i^{t+1}$. From Taylor-Lagrange inequality applied to $u\xrightarrow[]{}1/u$, and the two points $1/w_i^{t+1}$ and $1/w_i^t$ we have for all $i\in[K]$,
        \[w_i^{t+1}-w_i^t+(w_i^t)^2\pa{1/w_i^{t+1}-1/w_i^{t}}=(\tilde{w}_i^t)^3\pa{1/w_i^{t+1}-1/w_i^{t}}^2\]
    for some $\tilde{w}_i^t\in J_i^t$. Dividing the first term by $\tilde{w}_i^t$ and summing over $i$, we obtain, by definition of $\Psi$:
    \[\norm{w^{t+1}-w^t+\nabla^2(w^t)^{-1}\pa{\nabla\Psi(w^t)-\nabla\Psi(w^{t+1})}}_{\tilde{w}^t}= \norm{\nabla\Psi(w^t)-\nabla\Psi(w^{t+1})}^2_{\star,\tilde{w}^t}\]
    Then, from Lemma~\ref{lemma:double}, as $\tilde{w}_i^t\geq w_i^t/2$ for all $i\in [K]$ we obtain
    \[\norm{w^{t+1}-w^t+\nabla^2(w^t)^{-1}\pa{\nabla\Psi(w^t)-\nabla\Psi(w^{t+1})}}_{w^t}\leq 2\norm{\nabla\Psi(w^t)-\nabla\Psi(w^{t+1})}^2_{\star,\tilde{w}^t}\]
    and combined with the first inequality,
    \begin{equation}
        \norm{w^{t+1}-w^t+\eta^t\nabla^2\Psi(w^t)^{-1}\hat{p}^t}_{w^t}\leq 2\norm{\nabla\Psi(w^t)-\nabla\Psi(w^{t+1})}^2_{\star,\tilde{w}^t}\tag{$\star$}
    \end{equation}
    Now, for all $i$, we notice the existence of some $\overline{w}_i^t\in J_i^t$, again from Taylor Lagrange, such that, for all $i\in[K]$:
    \[1/w_i^{t+1}-1/w_i^t=1/\pa{{\overline{w}_i^t}}^2\pa{w_i^t-w_i^{t+1}}\]
    hence
    \[\nabla^2\Psi(\overline{w}^t)^{-1}\pa{\nabla\Psi(w^{t+1})-\nabla\Psi(w^{t})}=w^t-w^{t+1}\in W_0\]
    which gives, applying Lemma~\ref{lemma:pythagore},
    \[\norm{\nabla\Psi(w^t)+\nabla\Psi(w^{t+1})}^2_{\star,\overline{w}^t}=\min_{g\in G}\norm{\nabla\Psi(w^t)+\nabla\Psi(w^{t+1})-g}^2_{\star,\overline{w}^t}\, .\]

    Combining this with Lemma~\ref{lemma:double}, we obtain
    \begin{align*}
        \norm{\nabla\Psi(w^t)-\nabla\Psi(w^{t+1})}^2_{\star,\tilde{w}^t}&\leq 4\norm{\nabla\Psi(w^t)-\nabla\Psi(w^{t+1})}^2_{\star,\overline{w}^t}\\
        &\leq 4\norm{\eta^t \hat{F}_{\tau^t}(w^t)}^2_{\star,\overline{w}^t}\\
        &\leq 16\norm{\eta^t \hat{F}_{\tau^t}(w^t)}^2_{\star,w^t}=16\pa{\sigma'\eta^t}^2
    \end{align*}
    which yields with $(\star)$,
    \[\norm{w^{t+1}-w^t+\eta^t\nabla^2\Psi(w^t)^{-1}\hat{p}^t}_{w^t}\leq 32(\sigma'\eta^t)^2\]

    \eqref{eq:2} then follows using the convexity of the norm and Jensen inequality, as $\E_t\bra{\nabla^2\Psi(w^t)^{-1}\hat{p}^t}=\nabla^2\Psi(w^t)^{-1}p^t$.
    
    \eqref{eq:3} and \eqref{eq:4} follows from the triangle inequality as
    \[\norm{\nabla^2\Psi(w^t)^{-1}\hat{p}^t}_{w^t}=\norm{\hat{p}^t}_{\star,w^t}\leq \norm{\hat{F}_{\tau^t}(w^t)}_{\star,w^t}=\sigma'\]
and
\[\norm{\nabla^2\Psi(w^t)^{-1}p^t}_{\star,w^t}=\norm{p^t}_{\star,w^t}\, .\]
    where we also used $32\sigma'\eta^0\leq 1$ from Assumption~\ref{assum:sigma} for \eqref{eq:3}.
\end{proof}

\begin{proposition}\label{lemma:regularization_change}
    For any $t\geq T_0$ and $\xi,\xi'\in\R^K$ , the following inequality holds:
    \[\abs{\scal{\xi}{\xi'}_{\star,w^{t+1}}-\scal{\xi}{\xi'}_{\star,w^t}}\leq 8\sigma'\eta^t\norm{\xi}_{\star,w^t}\norm{\xi'}_{\star,w^t}\, .\]
    Furthermore, if $\xi$ and $\xi'$ are measurable with respect to $\cF^t$,
    \[\abs{\E_t\bra{\scal{\xi}{\xi'}_{\star,w^{t+1}}}-\scal{\xi}{\xi'}_{\star,w^t}}\leq\pa{\norm{ p^t}_{\star,w^t}\eta^t+36\pa{\sigma'\eta^t}^2} \norm{\xi}_{\star,w^t}\norm{\xi'}_{\star,w^t}\, .\]
\end{proposition}

To prove this proposition, we start with the following lemma that characterizes the variation of the squared coefficients that appear in the regularization.

\begin{lemma}\label{lemma:w_square}
    For all $i\in[K]$ and $t\geq T_0$:
    \[\abs{\pa{w^{t+1}_i}^2-\pa{w_i^t}^2}\leq 4\pa{w_i^t}^2\norm{w^{t+1}-w^t}_{w^t}\]
    and
    \[\abs{\E_t\bra{\pa{w_i^{t+1}}^2}-\pa{w_i^t}^2}\leq \pa{w_i^t}^2\pa{2\norm{\E_t\bra{w^{t+1}}-w^t}_{w^t}+\E_t\bra{\norm{w^{t+1}-w^t}_{w^t}^2}}\]
\end{lemma}
\begin{proof}
    For all $a,b\in\R$, the following equality holds:
    \[b^2-a^2=2a(b-a)+(b-a)^2\\.\]
    From Lemma~\ref{lemma:double}, we have
    \[\abs{\frac{w^{t+1}_i-w_i^t}{w_i^t}}\leq 2\, ,\]
    and with $a=w_i^t$ and $b=w^{t+1}_i$, the first equality yields
    \begin{align*}
        \abs{\pa{w^{t+1}_i}^2-\pa{w_i^t}^2}&=\abs{2w_i^t\pa{w^{t+1}_i-w_i^t}+(w_i^{t+1}-w_i^t)^2}\\
        &\leq 2w_i^t\abs{w^{t+1}_i-w_i^t}+\abs{w_i^{t+1}-w_i^t}^2\\
        &= \pa{w_i^{t}}^2\pa{2\abs{\frac{w^{t+1}_i-w_i^t}{w_i^t}}+\abs{\frac{w^{t+1}_i-w_i^t}{w_i^t}}^2}\\
        &\leq 4\pa{w_i^{t}}^2\abs{\frac{w^{t+1}_i-w_i^t}{w_i^t}}\\
        &\leq 4\pa{w_i^t}^2\norm{w^{t+1}-w^t}_{w^t}\, .
    \end{align*}
    For the second inequality, the idea is the same, but the second order term cannot be upper-bounded by the first:
    \begin{align*}
        \abs{\E_t\bra{\pa{w^{t+1}_i}^2}-\pa{w_i^t}^2}&=\abs{\E_t\bra{2w_i^t\pa{w^{t+1}_i-w_i^t}+(w_i^{t}-w_i^t)^2}}\\
        &\leq 2w_i^t\abs{\E_t\bra{w^{t+1}_i}-w_i^t}+\E_t\bra{\abs{w_i^{t+1}-w_i^t}^2}\\
        &= \pa{w_i^{t}}^2\pa{2\abs{\frac{\E_t\bra{w_i^{t+1}}-w_i^t}{w_i^t}}+\E_t\bra{\abs{\frac{w_i^{t+1}-w_i^t}{w_i^t}}^2}}\\
        &\leq \pa{w_i^t}^2\bra{2\norm{\E_t\bra{w^{t+1}}-w^t}_{w^t}+\E_t\bra{\norm{w^{t+1}-w^t}_{w^t}^2}}\, .
    \end{align*}
\end{proof}

Now, going to the proof of proposition~\ref{lemma:regularization_change}.
\begin{proof}
    With the first inequality of Lemma~\ref{lemma:w_square},
    \begin{align*}
        \abs{\scal{\xi}{\xi'}_{\star,w^{t+1}}-\scal{\xi}{\xi'}_{\star,w^t}}&=\abs{\sum_{i=1}^K \xi_i\xi'_i\bra{\pa{ w_i^{t+1}}^2-\pa{w_i^t}^2}}\\
        &\leq\sum_{i=1}^K\abs{\xi_i\xi'_i} \abs{\pa{w_i^{t+1}}^2-\pa{w_i^{t}}^2}\\
        &\leq 4\norm{w^{t+1}-w^t}_{w^t}\sum_{i=1}^k\abs{\xi_i\xi'_i}\pa{w_i^t}^2\\
        &\leq4\norm{w^{t+1}-w^t}_{w^t}\norm{\xi}_{\star,w^t}\norm{\xi'}_{\star,w^t}
    \end{align*}
    using Cauchy-Schwarz for the last inequality. The bound follows from \eqref{eq:3} of Proposition~\ref{lemma:w_inequalities} as
    \[\norm{w^{t+1}-w^t}_{w^t}\leq 2\sigma'\eta^t\]

    For the second inequality, with the second inequality of Lemma~\ref{lemma:w_square} and Cauchy-Schwarz:
    \begin{align*}
        &\abs{\E_t\bra{\scal{\xi}{\xi'}_{\star,w^{t+1}}}-\scal{\xi}{\xi'}_{\star,w^t}}\\
        & \qquad\qquad=\abs{\sum_{i=1}^K \xi_i\xi_i'\bra{\E_t\bra{\pa{ w_i^{t+1}}^2}-\pa{w_i^t}^2}}\\
        & \qquad\qquad\leq\sum_{i=1}^K \abs{\xi_i\xi_i'}\abs{\E_t\bra{\pa{ w_i^{t+1}}^2}-\pa{w_i^t}^2}\\
        & \qquad\qquad\leq\pa{2\norm{\E_t\bra{w^{t+1}_i}-w^t}_{w^t}+\E_t\bra{\norm{w^{t+1}-w^t}_{w^t}^2}}\sum_{i=1}^K\abs{\xi_i\xi'_i}\pa{w_i^t}^2\\
        & \qquad\qquad=\pa{2\norm{\E_t\bra{w^{t+1}}-w^t}_{w^t}+\E_t\bra{\norm{w^{t+1}-w^t}_{w^t}^2}}\norm{\xi}_{\star,w^t}\norm{\xi'}_{\star,w^t}
    \end{align*}
    and we conclude using \eqref{eq:4} of Proposition \ref{lemma:w_inequalities},
    \[\norm{\E_t\bra{w^{t+1}}-w^t}_{w^t}\leq \norm{\eta^t p^t}_{\star,w^t}+32\pa{\sigma'\eta^t}^2\]
    and with \eqref{eq:3} again:
    \[\E_t\bra{\norm{w^{t+1}-w^t}_{w^t}^2}\leq\pa{2\sigma'\eta^t}^2\leq 4(\sigma'\eta^t)^2\]
\end{proof}

\begin{proposition}\label{prop:rho}
    \[\norm{F_{\tau^{t+1}}(w^{t+1})-F_{\tau^t}(w^t)}_{\star,w^{t+1}}\leq \rho (t+T_0)^{-3/4}\]
    where
    \[\rho=\sigma'\eta\pa{4L+2\tau^0}+\frac{\tau^0}{4}\sqrt{K}\]
    In particular,
    \[\norm{F_{\tau^{t+1}}(w^{t+1})-F_{\tau^t}(w^t)}_{\star,w^t}\leq 2\rho (t+T_0)^{-3/4}\]
\end{proposition}

\begin{proof}
    We first notice that for any $w,w'\in W$ 
    {\small\[\norm{\nabla\Psi(w')-\nabla\Psi(w)}_{\star,w'}^2=\sum_{i=1}^K (w_i')^2\pa{-\frac{1}{w_i'}+\frac{1}{w_i}}^2=\sum_{i=1}^K \frac{1}{(w_i)^2}\pa{w_i'-w_i}^2=\norm{w'-w}_{w}
    \]}
    Now,
    {\small\begin{align*}
        &\norm{F_{\tau^{t+1}}(w^{t+1})-F_{\tau^t}(w^t)}_{\star,w^{t+1}}\\
        &\quad\leq \norm{F(w^{t+1}-F(w^t)}_{\star,w^t}+\tau^t\norm{\nabla\Psi(w^{t+1})-\nabla\Psi(w^t)}_{\star,w^{t+1}}+\abs{\tau^t-\tau^{t+1}}\norm{\nabla\Psi(w^{t+1})}_{\star,w^{t+1}}\\
        &\quad\leq L\norm{w^{t+1}-w^t}_{w^{t+1}}+\tau^t\norm{w^{t+1}-w^t}_{w^t}+\abs{\tau^{t}-\tau^{t+1}}\sqrt{\sum_{i=1}^K\pa{\frac{w_i^{t+1}}{w_i^{t+1}}}^2}\\
        &\quad\leq \pa{2L+\tau^0}\norm{w^{t+1}-w^t}_{w^t}+\abs{\tau^{t}-\tau^{t+1}}\sqrt{K}\\
        &\quad\leq 2\sigma'\eta^t\pa{2L+\tau^0}+\frac{\tau^0}{4(t+T_0)}\sqrt{K}\\
        &\quad\leq 2\sigma'\eta\pa{2L+\tau^0}(t+T_0)^{-\frac{3}{4}}+\frac{\tau^0}{4}\sqrt{K}(t+T_0)^{-\frac{3}{4}}
    \end{align*}}
    where we used, in this order:
    \begin{itemize}
        \item The triangle inequality
        \item Lemma~\ref{lemma:lipschitz} and the previous property
        \item Lemma~\ref{lemma:double} and the fact that the sequence $(\tau^t)_t$ is decreasing.
        \item Lemma~\ref{lemma:w_inequalities} and Lemma~\ref{lemma:delta_reg}
        \item The definition of $\eta^t$
    \end{itemize}
    The second inequality follows from Lemma~\ref{lemma:double}.

\end{proof}

\section{Proof - Recursive bound}

Let $d=(d^t)_{t\geq 1}$ be the sequence defined by
\[d^t:=d_{\tau^t}(w^t)=\norm{p^t}^2_{\star,w^t}\]

Remember that 
\[\norm{p^t}_{\star,w^t}=\min_{g\in G}\norm{F_\tau^t(w^t)-g}_{\star,w^t}=\norm{F_\tau^t(w^t)-g^t}_{\star,w^{t}}\]
Then, as $\norm{x+y}_{\star,w^{t+1}}^2=\norm{x}^2_{\star,w^{t+1}}+2\scal{x}{y}_{\star,w^{t+1}}+\norm{y}^2_{\star,w^{t+1}}$ for any $x,y\in\R^K$, with $x=p^t$ and $y=F_{\tau^t}(w^{t+1})-F_{\tau^t}(w^t)$
\begin{align*}
    d^{t+1}-d^t&\leq \norm{F_{\tau}^{t+1}(w^{t+1})-g^t}^2_{\star,w^{t+1}}-\norm{p^t}_{\star,w^t}\\
    &= \norm{p^t}^2_{\star,w^{t+1}}-\norm{p^t}^2_{\star,w^t}+2\scal{F_{\tau^{t+1}}(w^{t+1})-F_{\tau^t}(w^t)}{p^t}_{\star,w^{t+1}}\\
    &\qquad\qquad+\norm{F_{\tau^{t+1}}(w^{t+1})-F_{\tau^t}(w^t)}^2_{\star,w^{t+1}}
\end{align*}

Let 
\[\Delta_1^t=\norm{p^t}^2_{\star,w^{t+1}}-\norm{p^t}^2_{\star,w^t}\]
\[\alpha^t=2\scal{F_{\tau^{t+1}}(w^{t+1})-F_{\tau^t}(w^t)}{p^t}_{\star,w^{t}}\]
\[\Delta_2^t=\scal{F_{\tau^{t+1}}(w^{t+1})-F_{\tau^t}(w^t)}{p^t}_{\star,w^{t+1}}-\scal{F_{\tau^{t+1}}(w^{t+1})-F_{\tau^t}(w^t)}{p^t}_{\star,w^{t}}\]
\[V^t=\norm{F_{\tau^{t+1}}(w^{t+1})-F_{\tau^t}(w^t)}^2_{\star,w^{t+1}}\]

Then, we can separate the bound between the terms:
\[d^{t+1}-d^t\leq\Delta_1^t+\alpha^t+\Delta_2^t+V^t\]

\begin{lemma}\label{lemma:deviation_high_prob}
    For any $t\geq T_0$,
    \[\abs{\Delta_1^t-\E_t\bra{\Delta_1^t}}\leq 16(\sigma')^2\eta^t \sqrt{d^t}\]
    and
    \[\abs{\alpha^t-\E_t\bra{\alpha^t}}\leq 4\rho\sqrt{d^t}(t+T_0)^{-\frac{3}{4}}\]
    
\end{lemma}

\begin{proof}
    We use Lemma~\ref{lemma:regularization_change} for the first term:
    \begin{align*}
         \abs{\Delta_1^t}&=\abs{\norm{F_{\tau^t}(w^t)-g^t}^2_{\star,w^{t+1}}-\norm{F_{\tau^t}(w^t)-g^t}^2_{\star,w^t}}\\
         &\leq 8\sigma'\eta^t\norm{F_{\tau^t}(w^t)-g^t}^2_{\star,w^t}\\
         &\leq 8(\sigma')^2\eta^t \sqrt{d^t}\, .
    \end{align*}

    For $\alpha^t$, we obtain with Cauchy-Schwarz and Proposition~\ref{prop:rho}:
    \begin{align*}
        \abs{\alpha^t}&=\abs{2\scal{F_{\tau^{t+1}}(w^{t+1})-F_{\tau^t}(w^t)}{p^t}_{\star,w^{t}}}\\
        &\leq 2\norm{p^t}_{\star,w^t}\norm{F_{\tau^{t+1}}(w^{t+1})-F_{\tau^t}(w^t)}_{\star,w^t}\\
        &= 2\rho\sqrt{d^t}(t+T_0)^{-\frac{3}{4}}
    \end{align*}

    The two inequalities are then obtained using the fact that, for any random variable $X$,
    \[\abs{X}\leq c \quad \textrm{a.s.}\implies \abs{X-\E_t\bra{X}}\leq 2c \quad \textrm{a.s.}\, .\]
\end{proof}

\begin{lemma}\label{lemma:delta1_exp}
    \[\E_t\bra{\Delta_1^t}\leq \eta^t (d^t)^3+36(\sigma')^4(\eta^t)^2\]
\end{lemma}
\begin{proof}
    We use the second inequality of Lemma~\ref{lemma:regularization_change} and obtain:
    \begin{align*}
        \E_t\bra{\Delta_1^t}&=\E_t\bra{\norm{p^t}^2_{\star,w^{t+1}}}-\norm{p^t}^2_{\star,w^t}\\
        &\leq\norm{p^t}^2_{\star,w^t}\pa{\eta^t\norm{p^t}_{\star,w^t}+36\pa{\sigma'\eta^t}^2}\\
        &=d^t\pa{\eta^t\sqrt{d^t}+36\pa{\sigma'\eta^t}^2}\\
        &\leq \eta^t (d^t)^3+36(\sigma')^4(\eta^t)^2
    \end{align*}
    \end{proof}

\begin{lemma}\label{lemma:alpha_exp}
    \[\E_t\bra{\alpha^t}\leq -2\tau^{t+1}\eta^t d^t+\frac{\tau^t}{2}(t+T_0)\sqrt{Kd^t}+128L(\eta^t)^2(\sigma')^3\]
\end{lemma}
\begin{proof}
    We first divide $\E_t\bra{\alpha^t}$ into several terms:
    \begin{align*}
        \E_t\bra{\alpha^t}&=2\scal{\E_t\bra{F_{\tau^{t+1}}(w^{t+1})}-F_{\tau^t}(w^t)}{p^t}_{\star,w^{t}}\\
        &=2\scal{\E_t\bra{F(w^{t+1})}-F(w^t)}{p^t}_{\star,w^{t}}\\
        &\qquad +2\tau^{t+1}\scal{\E_t\bra{\nabla\Psi(w^{t+1})}-\nabla\Psi(w^t)}{p^t}_{\star,w^{t}}\\
        &\qquad+2(\tau^t-\tau^{t+1})\scal{\nabla\Psi(w^t)}{p^t}_{\star,w^{t}}
    \end{align*}

    \textbf{First term:} We start by multiplying this first term by $\eta^t$.
    \begin{align*}
        2\eta^t&\scal{\E_t\bra{F(w^{t+1})}-F(w^t)}{p^t}_{\star,w^{t}}\\
        &\quad=-2\scal{\E_t\bra{F(w^{t+1})-F(w^t)}}{\nabla^2\Psi(w^t)\E_t\bra{w^{t+1}-w^t}}_{\star,w^{t}}\\
        &\qquad\qquad+2\scal{\E_t\bra{F(w^{t+1})-F(w^t)}}{\nabla^2\Psi(w^t)\E_t\bra{w^{t+1}-w^t}+\eta^t p^t}_{\star,w^{t}}\\
        &\quad=-2\scal{\E_t\bra{F(w^{t+1})-F(w^t)}}{\pa{\E_t\bra{w^{t+1}-w^t}}} \\
        &\qquad\qquad+2\scal{\E_t\bra{F(w^{t+1})-F(w^t)}}{\E_t\bra{w^{t+1}-w^t}+\eta^t\nabla^2\Psi(w^t)^{-1}p^t} \\
        &\quad\leq 0 + 2\norm{F(w^{t+1})-F(w^t)}_{\star,w^t}\norm{\E_t\bra{w^{t+1}-w^t}+\eta^t\nabla^2\Psi(w^t)^{-1}p^t}_{w^t}\\
        &\quad\leq 2L\norm{w^{t+1}-w^t}_{w^t}\norm{\E_t\bra{w^{t+1}-w^t}+\eta^t\nabla^2\Psi(w^t)^{-1}p^t}_{w^t}\\
        &\quad\leq 128L\pa{\sigma'\eta^t}^3
    \end{align*}
    Where we used in this order:
    \begin{itemize}
        \item The fact that $F$ is a linear monotone operator, hence 
        \[\scal{\E_t \bra{F\pa{w^{t+1}}-F\pa{w^t}}}{\E_t\bra{w^{t+1}}-w^t} \geq 0\]
        \item Cauchy-Schwarz as
        $\scal{x}{y} =\scal{\nabla^2\Psi(w)^{1/2}x}{\nabla^2\Psi(w)^{-1/2}y} \leq \norm{x}_{w}\norm{y}_{\star,w}$
        \item Assumption~\ref{lemma:lipschitz}: $\norm{F(w^{t+1})-F(w^t)}_{\star,w^t}\leq L\norm{w^{t+1}-w^t}_{w^t}$
        \item Inequalities \eqref{eq:2} and \eqref{eq:3} of Lemma~\ref{lemma:w_inequalities}
    \end{itemize}
    Dividing by $\eta^t$, we then obtain
    \[2\scal{\E_t\bra{F(w^{t+1})}-F(w^t)}{p^t}_{\star,w^{t}}\leq 128L(\eta^t)^2(\sigma')^3\]

    \textbf{Second term:} We use the fact that 
    \[\nabla\Psi(w^{t+1})-\nabla\Psi(w^t)+\eta^t\hat{p}^t\in G\]
    and hence, taking the expectation with respect to $\cF^t$,
    \[\E_t\bra{\nabla\Psi(w^{t+1})}-\nabla\Psi(w^t)+\eta^t p^t\in G\, .\]
    From Lemma~\ref{lemma:pythagore}, as $\nabla^2\Psi(w^t)^{-1}p^t\in W_0$ this yields
    \[2\tau^{t+1}\scal{\E_t\bra{\nabla\Psi(w^{t+1})}-\nabla\Psi(w^t)}{p^t}_{\star,w^{t}}=-2\tau^{t+1}\scal{\eta^tp^t}{p^t}_{\star,w^t}=-2\tau^{t+1}\eta^t d^t\]

    \textbf{Third term:} From Cauchy-Schwarz again,
    \begin{align*}
        2(\tau^t-\tau^{t+1})\scal{\nabla\Psi(w^t)}{p^t}_{\star,w^{t}}&\leq 2(\tau^t-\tau^{t+1})\norm{\nabla\Psi(w^t)}_{\star,w^t}\norm{p^t}_{\star,w^t}\\
        &=2(\tau^t-\tau^{t+1})\sqrt{Kd^t}\\
        &\leq \frac{\tau^t}{2}(t+T_0)\sqrt{Kd^t}
    \end{align*}
    where we used Lemma~\ref{lemma:delta_reg} for the last inequality.
\end{proof}

\begin{lemma} \label{lemma:absolute_upperbound}
    \[\Delta_2^t\leq 8\rho(\sigma')^2\eta^t(t+T_0)^{-\frac{3}{4}}\]
    \[V^t\leq \rho^2(t+T_0)^{-\frac{3}{2}}\]
\end{lemma}

\begin{proof}
    From Proposition~\ref{lemma:regularization_change} and Proposition~\ref{prop:rho}, we obtain
    \begin{align*}
        \Delta_2^t&=\scal{F_{\tau^{t+1}}(w^{t+1})-F_{\tau^t}(w^t)}{p^t}_{\star,w^{t+1}}-\scal{F_{\tau^{t+1}}(w^{t+1})-F_{\tau^t}(w^t)}{p^t}_{\star,w^{t}}\\
        &\leq 4\sigma'\eta^t\norm{F_{\tau^{t+1}}(w^{t+1})-F_{\tau^t}(w^t)}_{\star,w^t}\norm{p^t}_{\star,w^t}\\
        &\leq 8\sigma'\eta^t\rho\sqrt{d^t}\rho(t+T_0)^{-\frac{3}{4}}\\
        &\leq 8\rho(\sigma')^2\eta^t(t+T_0)^{-\frac{3}{4}}
    \end{align*}
    The second inequality directly follows from Proposition~\ref{prop:rho} .   
\end{proof}

\section{Proof - Final bound} \label{sec:final_bound}

Remember that 
\begin{equation}
    d^{t+1}-d^t\leq \Delta_1^t+\alpha^t+\Delta_2^t+V^t \label{divide_update}
\end{equation}

In this section, we use $s=t+T_0$

Let, for all $t$ and associated $s$,
\[U^t=\frac{d^t}{\log(s/\delta).\tau^2},\]
\[b^{t+1}=\pa{\E_t\bra{\alpha^t+\Delta^t_1}+\textrm{ess sup} \bra{\Delta_2^t+V^t \middle| \cF^t}+\frac{d^t}{s}}\frac{1}{\log(s/\delta).\tau^2}\]
and
\[W^{t+1}=\pa{\alpha^t+\Delta^t_1-\E_t\bra{\alpha^t+\Delta^t_1}}\frac{1}{\log(s/\delta).\tau^2}\]

From dividing \eqref{divide_update} by ${\log(s\delta)}.{\tau^2}$, on both side, we obtain
\begin{align*}
    U^{t+1}&\leq \frac{d^{t+1}}{\log(s/\delta).\tau^2}\\
    &\leq\pa{d^t+\alpha+\Delta_1^t+\Delta_2^t+V^t}\frac{1}{\log(s/\delta).\tau^2}\\
    &\leq \pa{1-\frac{1}{s}}U^t +b^{t+1}+W^{s+1}
\end{align*}

We need to show that $U^{T_0}$, $b$, and $W$ satisfy the conditions of Lemma~\ref{lemma:martingale} for correctly chosen parameters.

For this purpose, we remind the following assumption:

\assparameter*

The third assumption is used to ensure that $\eta^0\tau^0\leq 1$ (for the algorithm to be well defined), and $U^{T_0}\leq \frac{\log(T_0)}{\sqrt{T_0}}$ for the starting condition of Lemma~\ref{lemma:martingale} to be satisfied.

\begin{lemma}
    In this context, at each iteration $t\leq T_1$, the $b$ sequence satisfies:
    \[b^{t+1}\leq (s+1)^{-3/2}\]
\end{lemma}

\begin{proof}
From Lemma~\ref{lemma:alpha_exp},

\begin{align*}
    \frac{1}{\log(s/\delta).\tau^2}&\E_t\bra{\alpha^t}\\
    &=\pa{-2\tau^{t+1}\eta^t d^t+\frac{\tau^t}{2s}\sqrt{Kd^t}+128L(\eta^t)^2(\sigma')^3}\frac{1}{\log(s/\delta).\tau^2}
\end{align*}

Then, decomposing each term

\begin{itemize}
    \item As $\frac{1}{\eta\tau}=o(1)$,
    \[-2\tau^{t+1}\eta^t d^t\frac{1}{\log(s/\delta)}\leq -\pa{\frac{5}{4}\tau^t\eta^t d^t+\frac{d^t}{s}}\frac{1}{\log(s/\delta)}\]
    \item From the inequality, for any $x,y>0$, $x\leq \frac{xy}{2}+\frac{x}{2y}$
    \begin{align*}
        \frac{\tau^t}{2s}\sqrt{Kd^t}\frac{1}{\log(s/\delta).\tau^2}&\leq\frac{\eta^t\tau^t d^t}{4}\frac{1}{\log(s/\delta).\tau^2}+\frac{\tau^t}{s\eta^t}\frac{K}{\log(s/\delta).\tau^2}\\
        &\leq \frac{\eta^t\tau^t d^t}{4}\frac{1}{\log(s/\delta).\tau^2}+\frac{s^{-3/2}}{4}
    \end{align*}
    where we also used $\frac{1}{\eta\tau}=o(1)$ for the last inequality
    \item Using $\eta=o(1)$, $\frac{1}{\tau}=o(1)$ and $\log(s/\delta)$,
    \[128L(\eta^t)^2(\sigma')^3\frac{1}{\log(s/\delta)\tau^2}\leq \frac{(s+1)^{-3/2}}{4}\]
\end{itemize}

Then, from Lemma~\ref{lemma:delta1_exp}:
\[\frac{1}{\log(s/\delta).\tau^2}\E_t\bra{\Delta_1^t}\leq\pa{\eta^t(d^t)^3+36(\sigma')^4(\eta^t)^2}\frac{1}{\log(s/\delta).\tau^2}\]

Again, by considering each term,

\begin{itemize}
    \item Using the recursive assumption
    \[\frac{1}{\log(s/\delta).\tau^2}\eta^t(d^t)^3\leq \frac{1}{\log(s/\delta).\tau^2}\eta^t d^t \leq \eta^t\tau^t d^t\]
    \item As $\eta=o(1)$ and $\frac{1}{\tau}=o(1)$ again,
    \[36(\sigma')^4(\eta^t)^2\frac{1}{\log(s/\delta).\tau^2}\leq \frac{(s+1)^{-3/2}}{4}\]
\end{itemize}

Finally, from Lemma~\ref{lemma:absolute_upperbound}

\begin{itemize}
    \item As $\eta=o(1)$ and $\frac{1}{\tau}=o(1)$
    \begin{align*}
        \frac{1}{\log(s/\delta).\tau^2}\Delta_2^t\leq 8\rho(\sigma')^2\eta^t s^{-\frac{3}{4}}\frac{1}{\log(s/\delta).\tau^2}\leq \frac{(s+1)^{-{3/2}}}{4}
    \end{align*}
    \item As $\frac{1}{\tau}=o(1)$ and $\log(s/\delta)>1$,
    \[\frac{1}{\log(s/\delta).\tau^2}V^t\leq \rho^2 s^{-\frac{3}{2}}\frac{1}{\log(s/\delta).\tau^2}\leq \frac{(s+1)^{-{3/2}}}{4}\]
\end{itemize}

Summing all these terms, we obtain that the $d^t$ terms compensate with each other, implying
\[b^{t+1}\leq (s+1)^{-3/2}.\]
\end{proof}

\begin{lemma}
Under this setting, at each iteration $t$, the $W$ sequence is bounded by
    \[\abs{W^{t+1}}^2\leq \frac{(s+1)^{-\frac{3}{2}}}{2}U^t\]
\end{lemma}

\begin{proof}
    From the first inequality of Lemma~\ref{lemma:deviation_high_prob}
    \begin{align*}
        \frac{1}{\log(s/\delta).\tau^2}\abs{\Delta_1^t-\E_t\bra{\Delta_1^t}}&\leq 16(\sigma')^2\eta^t \sqrt{d^t}\frac{1}{\log(s/\delta).\tau^2}\\
        &\leq \frac{(s+1)^{-3/4}}{4}\sqrt{U^t}
    \end{align*}
    where we used $\eta=o(1)$, $\tau>1$ and the definition of $U^t$.
    
    Similarly, with the second equality of the lemma, we also get
    \begin{align*}
        \frac{1}{\log(s/\delta).\tau^2}\abs{\alpha^t-\E_t\bra{\alpha^t}}&\leq 4\rho\sqrt{d^t}s^{-\frac{3}{4}}\frac{1}{\log(s/\delta).\tau^2}\\
        &\leq 4\rho\sqrt{U^t}\frac{s^{-3/4}}{\tau}\\
        &\leq \frac{(s+1)^{-3/4}}{4}\sqrt{U^t}\, .
    \end{align*}
    Summing the two inequalities above and squaring yields the result.
\end{proof}

These two lemmas are finally used for the proof of the main theorem

\thmlast*

\begin{proof}
    Using the two lemmas above along with the parameters defined by Assumption~\ref{ass:parameters}, we obtain that the sequence $U$ satisfies the assumptions of Lemma~\ref{lemma:martingale}.
    
    Hence, at all iterations $t$,
    \[\frac{d^t}{\log(s/\delta)\tau^2}\leq \sqrt{s}\log(s/\delta)\]
    Equivalent to
    \[d^t\leq (\tau^t)^2\]
    This lets us apply Lemma~\ref{lemma:dual_gap}, which yields
    \[\EG(w^t)\leq 2\tau^t K\, .\]
\end{proof}

\thmlastext*

\begin{proof}
    The proof is the same as above.
\end{proof}

\end{document}